\documentclass[letterpaper, 10 pt, conference]{ieeeconf}
\IEEEoverridecommandlockouts                              
\usepackage{multicol}
\usepackage[bookmarks=true]{hyperref}
\usepackage{amsmath} 
\usepackage{amssymb}  
\usepackage{mathtools}
\usepackage{xcolor} 
\usepackage{ulem} 
\usepackage{multirow}
\usepackage{algorithm}
\usepackage{algpseudocode}

\usepackage{color}
\usepackage{latexsym}
\usepackage{multicol}

\usepackage[font=small,labelfont=bf]{caption}
\usepackage{lipsum}
\usepackage{subcaption}



\renewcommand{\thefigure}{\arabic{figure} }

\newtheorem{definition}{Definition}
\newtheorem{problem}{Problem}
\newtheorem{theorem}{Theorem}

\newtheorem{lemma}[theorem]{Lemma}

\newtheorem{example}{Example}
\newcommand*{\mydprime}{^{\prime\prime}\mkern-1.2mu}

\usepackage[subnum]{cases}

\newcommand{\Next}{\bigcirc}
\newcommand{\Always}{\Box}
\newcommand{\Event}{\diamondsuit}
\newcommand{\Implies}{\Rightarrow}
\newcommand{\Then}{\mathcal{T}}
\newcommand{\True}{\top}

\newcommand{\argmax}{\textrm{argmax}}
\newcommand{\mami}{\textrm{mami}}

\newcommand{\ignore}[1]{%
}

\begin{document}

\title{A Policy Search Method For Temporal Logic Specified Reinforcement Learning Tasks}

\author{Xiao Li, Yao Ma and Calin Belta}



%

\normalem 

\maketitle

\begin{abstract}
Reward engineering is an important aspect of reinforcement learning. Whether or not the users' intentions can be correctly encapsulated in the reward function can significantly impact the learning outcome. Current methods rely on manually crafted reward functions that often requires parameter tuning to obtain the desired behavior. This operation can be expensive when exploration requires systems to interact with the physical world. In this paper, we explore the use of {\em temporal logic} (TL) to specify tasks in reinforcement learning. TL formula can be translated to a real-valued function that measures its level of satisfaction against a trajectory. We take advantage of this function and propose {\em temporal logic policy search} (TLPS), a model-free learning technique that finds a policy that satisfies the TL specification. A set of simulated experiments are conducted to evaluate the proposed approach.  
\end{abstract}

\IEEEpeerreviewmaketitle

\section{Introduction}
\label{sec:intro}
Reinforcement learning (RL) has enjoyed groundbreaking success in recent years ranging from playing Atari games at super-human level~\cite{Mnih2015}, playing competitively with world champions in the game of Go \cite{Silver2016} to generating visuomotor control policies for robots \cite{Levine2015}, \cite{Levine2016}. Despite much effort being put into developing sample efficient algorithms, an important aspect of RL remains less explored. The reward function is the window for designers to specify the desired behavior and impose important constraints for the system. While most reward functions used in the current RL literature have been  based on heuristics for relatively simple tasks, real world applications typically involve tasks that are logically more complex.

Commonly adopted reward functions take the form of a linear combination of basis functions (often quadratic)~\cite{gu2016deep}. This type of reward function has limited expressibility and is semantically ambiguous because of its dependence on a set of weights. Reward functions of this form have been used to successfully learn high dimensional control tasks such as humanoid walking~\cite{peters2003reinforcement} and multiple household tasks (e.g. placing coat-hangers, twisting bottle caps, etc)~\cite{Levine2015}. However, parameter tuning of the reward function is required and this iteration is expensive for robotic tasks. Moreover, these tasks are logically straightforward in that there is little logical interactions between sub-tasks (such as sequencing, conjunction/disjunction, implication, etc). 

Consider the harder task of learning to use an oven. The agent is required to perform a series of sub-tasks in the correct sequence (set temperature and timer $\rightarrow$ preheat $\rightarrow$ open oven door $\rightarrow$ place item in oven $\rightarrow$ close oven door). In addition, the agent has to make the simple decision of when to open the oven door and place the item (i.e. preheat finished \textit{implies} open oven door). Tasks like this are commonly found in household environments (using the microwave, refrigerator or even a drawer) and a function that correctly maps the desired behavior to a real-valued reward can be difficult to design. If the semantics of the reward function can not be guaranteed, then an increase in the expected return will not necessarily represent better satisfaction of the task specification. This is referred to as reward hacking by~\cite{Amodei2016}.

Reward engineering has been briefly explored in the reinforcement learning literature. Authors of \cite{dewey2014reinforcement} and \cite{arel2012threat} provide general formalisms for reward engineering and discuss its significance. Authors of~\cite{Ng1999} proposed potential-based reward shaping and proved policy invariance under this type of reward transformation.  Another line of work aims to infer a reward function from demonstration. This idea is called inverse reinforcement learning and is explored by \cite{Ng2000} and \cite{sermanet2016unsupervised}.

In this paper, we adopt the expressive power of temporal logic and use it as a task specification language for reinforcement learning in continuous state and action spaces. Its quantitative semantics (also referred to as robustness degree or simply robustness) translate a TL formula to a real-valued function that can be used as the reward. By definition of the quantitative semantics, a robustness value of greater than zero \textit{guarantees} satisfaction of the temporal logic specification.

Temporal logic (TL) has been adopted as the specification language for a wide variety of control tasks. Authors of~\cite{leahy2016persistent} use linear temporal logic (LTL) to specify a persistent surveillance task carried out by aerial robots. Similarly, \cite{sadraddini2016provably} and \cite{coogan2016traffic} applied LTL in traffic network control. Application of TL in reinforcement learning has been less investigated. \cite{Aksaray2016} combined signal temporal logic (STL) with Q-learning while also adopting the log-sum-exp approximation of robustness. However, their focus is in the discrete state and action spaces, and ensured satisfiability by expanding the state space to a history dependent state space. This does not scale well for large or continuous state-action spaces which is often the case for control tasks. 

Our main contributions in this paper are as follows:
\begin{itemize}
\item we present a model-free policy search algorithm, which we call temporal logic policy search (TLPS), that takes advantage of the robustness function to facilitate learning. We show that an optimal parameterized policy that maximizes the robustness could be obtained by solving a constrained optimization,
\item a smoothing approximation of the robustness degree is proposed which is necessary for obtaining the gradients of the objective and constraints. We prove that using the smoothed robustness as reward provides similar semantic guarantees to the original robustness definition while providing significant speedup in learning, 
\item  finally, we demonstrate the performance of the proposed approach using simulated navigation tasks.  
\end{itemize}



\section{Preliminaries}

\subsection{Truncated Linear Temporal Logic (TLTL)}
\label{sec:tltl}
In this section, we provide definitions for TLTL (refer to our previous work~\cite{li2016reinforcement} for a more elaborate discussion of TLTL).  A TLTL formula is defined over predicates of form $f(s) < c$,
where $f: {\rm I\!R}^n \rightarrow {\rm I\!R}$ is a function of state and $c$ is a constant. We express the task as a TLTL formula with the following syntax:
\begin{equation}
\begin{split}
\phi := \ & \True \,\,|\,\, f(s) < c \,\,| \,\, \neg \phi \,\,|\,\, \phi \wedge \psi \,\,|\,\, \phi \vee \psi \,\,|\,\, \\
        & \Event \phi \,\,|\,\, \Always \phi \,\,|\,\, \phi \, \mathcal{U} \, \psi \,\,|\,\, \phi\, \Then\, \psi \,\,|\,\, \Next \phi \,\,|\,\, \phi \Implies \psi,
\end{split}
\end{equation}
where $\True$ is the boolean constant true, $f(s) < c$ is a predicate,
$\neg$~(negation/not), $\wedge$~(conjunction/and), and $\vee$~(disjunction/or) are
Boolean connectives,
and $\Event$~(eventually), $\Always$~(always), $\mathcal{U}$~(until), $\Then$~(then), $\Next$~(next),
are temporal operators.
Implication is denoted by $\Implies$~(implication). TLTL formulas are evaluated against finite time sequences of states $\{s_{0},s_{1},\ldots,s_{T}
\}$. 

We denote $s_t \in S$ to be the state at time $t$, and $s_{t:t+k}$ to be a sequence of states
(state trajectory) from time $t$ to $t+k$, i.e., $s_{t:t+k}=(s_t, s_{t+1}, ..., s_{t+k})$. The Boolean semantics of TLTL is defined as:
\begin{alignat*}{3}
&s_{t:t+k} \models f(s)<c \quad &&\Leftrightarrow \quad &&f(s_t) <c, \\
&s_{t:t+k} \models \neg \phi \quad &&\Leftrightarrow \quad &&\neg(s_{t:t+k}\models \phi),\\
&s_{t:t+k} \models \phi \Rightarrow \psi  \quad &&\Leftrightarrow \quad && (s_{t:t+k} \models \phi) \Rightarrow (s_{t:t+k} \models \psi),\\
&s_{t:t+k} \models \phi \wedge \psi \quad &&\Leftrightarrow \quad && (s_{t:t+k} \models \phi) \wedge (s_{t:t+k} \models \psi),\\
&s_{t:t+k} \models \phi \vee \psi \quad &&\Leftrightarrow \quad && (s_{t:t+k} \models \phi) \vee (s_{t:t+k} \models \psi),\\
&s_{t:t+k} \models \Next \phi  \quad &&\Leftrightarrow \quad && (s_{t+1:t+k} \models \phi) \wedge (k>0), \\
&s_{t:t+k} \models \Always \phi \quad &&\Leftrightarrow \quad && \forall t^\prime \in [t,t+k) \  s_{t^\prime:t+k} \models \phi,\\
&s_{t:t+k} \models \Event \phi \quad &&\Leftrightarrow \quad && \exists t^\prime \in [t,t+k) \ s_{t^\prime:t+k} \models \phi,\\
&s_{t:t+k} \models \phi \,\, \mathcal{U} \,\, \psi \quad &&\Leftrightarrow \quad && \exists t^\prime \in [t,t+k) \,\,s.t.\,\, s_{t^\prime:t+k} \models \psi \\
&\,&&\,&& \wedge (\forall t^{\prime\prime} \in [t,t^\prime) \ s_{t^{\prime\prime}:t^\prime} \models \phi),\\
&s_{t:t+k} \models \phi \,\, \mathcal{T}\,\, \psi \quad &&\Leftrightarrow \quad && \exists t^\prime \in [t,t+k) \,\,s.t.\,\, s_{t^\prime:t+k} \models \psi \\
&\,&&\,&& \wedge (\exists t^{\prime\prime} \in [t,t^\prime) \ s_{t^{\prime\prime}:t^\prime} \models \phi).
\end{alignat*}
Intuitively, state trajectory $s_{t:t+k}\models \Always \phi$ (reads $s_{t:t+k}$ satisfies $\Always \phi$) if the specification defined by
$\phi$ is satisfied for every subtrajectory $s_{t^\prime:t+k},\,\,t^\prime \in [t,t+k)$.
Similarly, $s_{t:t+k}\models \Event \phi$ if $\phi$ is satisfied for at least one subtrajectory
$s_{t^\prime:t+k},\,\,t^\prime \in [t,t+k)$.
$s_{t:t+k}\models \phi \,\, \mathcal{U} \,\, \psi$ if $\phi$ is satisfied at every time step before
$\psi$ is satisfied, and $\psi$ is satisfied at a time between $t$ and $t+k$.
$s_{t:t+k}\models \phi \,\, \Then \,\, \psi$ if $\phi$ is satisfied at least once before
$\psi$ is satisfied between $t$ and $t+k$.
A trajectory $s$ of duration $k$ is said to satisfy formula $\phi$ if $s_{0:k} \models \phi$.

TLTL is equipped with quantitative semantics (robustness degree)
, i.e., a real-valued function $\rho(s_{t:t+k}, \phi)$ that indicates how far $s_{t:t+k}$ is from satisfying or
violating the specification $\phi$. We define the task satisfaction measurement $\rho(\tau,\phi)$
, which is recursively expressed as:
\begin{alignat*}{3}
&\rho(s_{t:t+k}, \True)\quad && = \quad && \rho_{max},\\
&\rho(s_{t:t+k},f(s_t)<c) \quad && = \quad &&c-f(s_t),\\
&\rho(s_{t:t+k},\neg \phi) \quad && = \quad &&-\rho(s_{t:t+k},\phi),\\
&\rho(s_{t:t+k}, \phi \,\, \Rightarrow \psi) \quad && = \quad && \max(-\rho(s_{t:t+k}, \phi), \rho(s_{t:t+k}, \psi))\\
&\rho(s_{t:t+k},\phi_1\wedge \phi_2) \quad &&= \quad &&\min(\rho(s_{t:t+k},\phi_1),\rho(s_{t:t+k},\phi_2)), \\
&\rho(s_{t:t+k},\phi_1\vee \phi_2) \quad &&= \quad &&\max(\rho(s_{t:t+k},\phi_1),\rho(s_{t:t+k},\phi_2)),\\
&\rho(s_{t:t+k}, \Next \phi) \quad && = \quad && \rho(s_{t+1:t+k},\phi) \,\,(k>0), \\
&\rho(s_{t:t+k},\Always \phi) \quad &&= \quad && \underset{t^{\prime} \in [t,t+k)}{\min}(\rho(s_{t^{\prime}:t+k},\phi)),\\
&\rho(s_{t:t+k},\Event \phi) \quad &&= \quad && \underset{t^{\prime} \in [t,t+k)}{\max}(\rho(s_{t^{\prime}:t+k},\phi)),\\
&\rho(s_{t:t+k},\phi \,\, \mathcal{U} \,\, \psi) \quad && = \quad && \underset{t^{\prime} \in [t,t+k)}{\max}( \min (\rho(s_{t^{\prime}:t+k},\psi), \\
& \, &&\, && \underset{t^{\mydprime} \in [t,t^{\prime})}{\min}\rho(s_{t\mydprime:t^\prime},\phi))),\\
&\rho(s_{t:t+k},\phi \,\, \Then \,\, \psi) \quad && = \quad && \underset{t^{\prime} \in [t,t+k)}{\max}( \min (\rho(s_{t^{\prime}:t+k},\psi), \\
& \, &&\, && \underset{t^{\mydprime} \in [t,t^{\prime})}{\max}\rho(s_{t\mydprime:t^\prime},\phi))),
\end{alignat*}
where $\rho_{max}$ represents the maximum robustness value.
Moreover, $\rho(s_{t:t+k},\phi) > 0 \Rightarrow s_{t:t+k} \models \phi$ and
$\rho(s_{t:t+k},\phi) < 0 \Rightarrow s_{t:t+k} \not\models \phi$,
which implies that the robustness degree can substitute Boolean semantics in order to enforce
the specification $\phi$. \newline 

\begin{example}
Consider
specification $\phi = \Event(s> 5 \wedge s < 10)$, where $s$ is a one dimensional state. Intuitively, this formula specifies that $s$ eventually reaches region $(5,10)$ for at least one time step.
Suppose we have a state trajectory $s_{0:3}=s_0s_1s_2=[11,6, 7]$ of horizon 3.
The robustness is
$\rho(s_{0:3}, \phi) = \underset{t \in [0,3)}{\max}\Big (\min(10 - s_t, s_t - 5)\Big) = \max(-1,1, 2) = 2$. Since $\rho(s_t, \phi) > 0$, $s_{0:1} \models \phi$ and the value $\rho(s_t, \phi) = 2$ is a measure of the satisfaction margin. Note that both states $s_1$ and $s_2$ stayed within the specified region, but $s_2$ "more" satisfies the predicate $(s> 5 \wedge s < 10)$ by being closer to the center of the region and thereby achieving a higher robustness value than $s_1$. 
\end{example}

\subsection{Markov Decision Process}
\label{sec:MDP}

In this section, we introduce the finite horizon infinite Markov decision process (MDP) and the semantics of a TLTL formula over an MDP. We start with the following definition:

\begin{definition}
\label{def:2}
A finite horizon infinite MDP is defined as a tuple $\langle S,A,p(\cdot|\cdot,\cdot)\rangle$, where $S\subseteq {\rm I\!R}^n$ is the continuous state space; $A \subseteq {\rm I\!R}^m$ is the continuous action space; $p(s^{\prime}|s,a)$ is the conditional probability density of taking action $a \in A$ at state $s \in S$ and ending up in state $s^{\prime} \in S$. We denote $T$ as the horizon of MDP.  
\end{definition}
 
Given an MDP in Definition~\ref{def:2}, a state trajectory of length $T$ (denoted $\tau=s_{0:T-1} = (s_0, ..., s_{T-1})$) can be produced. The semantics of a TLTL formula $\phi$ over $\tau$ can be evaluated with the robustness degree $\rho(\tau, \phi)$ defined in the previous section. $\rho(\tau, \phi) > 0$ implies that $\tau$ satisfies $\phi$, i.e. $\tau \models \phi$ and vice versa. In the next section, we will take advantage of this property and propose a policy search method that aims to maximize the expected robustness degree.

\section{Problem Formulation And Approach}
\label{sec:problem_formulation}
We first formulate the problem of policy search with TLTL specification as follows:
\begin{problem}
\label{def:1}

Given an MDP in Definition~\ref{def:2} and a TLTL formula $\phi$, find a stochastic policy $\pi(a|s)$ ($\pi$ determines a probability of taking action $a$ at state $s$) that maximizes the expected robustness degree 

\begin{equation}
\label{eq:2}
\pi^\star = \underset{\pi}{\arg\max} \, E_{p^\pi(\tau)}\left[ \rho(\tau, \phi)\right],
\end{equation}
where the expectation is taken over the trajectory distribution $p^\pi(\tau)$ following policy $\pi$, i.e.

\begin{equation}
\label{eq:11}
p^\pi(\tau) = p(s_0)\prod_{t=0}^{T-1}p(s_{t+1}|s_t,a_t)\pi(a_t|s_t).
\end{equation} \newline

\end{problem}

In reinforcement learning, the transition function $p(s^{\prime}|s,a)$ is unknown to the agent. The solution to Problem \ref{def:1} learns a stochastic time-varying policy $\pi(a_t|s_t)$ \cite{Deisenroth2011} which is a conditional probability density function of action $a$ given current state $s$ at time step $t$. 

In this paper, policy $\pi$ is a parameterized policy $\pi(a_t|s_t;\theta_t),\forall t=1,\ldots,T$ (also written as $\pi_\theta$ in short, where $\theta=\{\theta_0,\theta_1,\ldots,\theta_{T-1}\}$) is used to represent the policy parameter. The objective defined in Equation~\eqref{eq:2} then becomes finding the optimal policy parameter $\theta^*$ such that

\begin{equation}
\label{eq:3}
\theta^\star = \underset{\theta}{\arg\max} \, E_{p^{\pi_\theta}(\tau)}\left[ \rho(\tau,\phi)\right].
\end{equation}

To solve Problem~\ref{def:1}, we introduce temporal logic policy search (TLPS) - a model free RL algorithm. At each iteration, a set of sample trajectories are collected under the current policy. Each sample trajectory is updated to a new one with higher robustness degree by following the gradient of $\rho$ while also keeping close to the sample so that dynamics is not violated. A new trajectory distribution is fitted to the set of updated trajectories. Each sample trajectory is then assigned a weight according to its probability under the updated distribution. Finally, the policy is updated with weight maximum likelihood. This process ensures that each policy update results in a trajectory distribution with higher expected robustness than the current one. Details of TLPS will be discussed in the next section. 

As introduced in Section~\ref{sec:tltl}, the robustness degree $\rho$ consists of embedded $\max/\min$ functions and calculating the gradient is not possible. In Section~\ref{sec:smoothing}, we discuss the use of \textit{log-sum-exp} to approximate the robustness function and provide proofs of some properties of the approximated robustness.

\section{Temporal Logic Policy Search (TLPS)}
\label{sec:tlps}

Given a TLTL formula $\phi$ over predicates of $S$, TLPS finds the parameters $\theta$ of a parametrized stochastic policy $\pi_\theta(a|s)$ that maximizes the following objective function.
\begin{equation}
\label{eq:6}
J^{\pi_\theta} = E_{p^{\pi_\theta}} [\rho(\tau,\phi)], \,\, (T < \infty),
\end{equation}
\noindent where $p^{\pi_\theta}=p^{\pi_\theta}(\tau)$ is defined in Equation~\eqref{eq:11}.

In TLPS, we model the policy as a time-varying linear Gaussian, i.e.  $\pi(a_t|s_t)=\mathcal{N}(K_ts_t+k_t,C_t)$ where $K_t, k_t, C_t$ are the feedback gain, feed-forward gain and covariance of the policy at time $t$. (similar approach has been adopted in~\cite{chebotar2016path},~\cite{montgomery2016guided}). And the trajectory distribution in Equation~\eqref{eq:11} is modeled as a Gaussian $p^{\pi_\theta}(\tau)=\mathcal{N}(\tau | \mu_\tau, \Sigma_\tau)$ where $\mu_\tau=(\mu_{s_0},...,\mu_{s_T})$ and $\Sigma_\tau=diag(\Sigma_{s_0},...,\Sigma_{s_T})$.

At each iteration, $N$ sample trajectories are collected (denoted $\tau^i$, $i \in [1,N]$). For each sample trajectory $\tau^i$, we find an updated trajectory $\bar{\tau}^i$ by solving

\begin{equation}\label{tlps-eq-1}
\underset{\bar{\tau}^i}{\max}\,\hat{\rho}(\bar{\tau}^i,\phi), \,\, s.t.\,\,(\bar{\tau}^i - \tau^i)^T(\bar{\tau}^i - \tau^i) < \epsilon.
\end{equation}
\noindent In the above equation, $\hat{\rho}$ is the \textit{log-sum-exp} approximation of $\rho$. This is to take advantage of the many off-the-shelf nonlinear programming methods that require gradient information of the Lagrangian (sequential quadratic programming is used in our experiments). Using the log-sum-exp approximation we can show that its approximation error is bounded. In additional, the local ascending directions on the approximated surface coincide with the actual surface given mild constraints (these will be discussed in more detail in the next section).  Equation~\eqref{tlps-eq-1} aims to find a new trajectory that achieves higher robustness. The constraint is to limit the deviation of the updated trajectory from the sample trajectory so the system dynamics is not violated. 

After we obtain a set of updated trajectories, an updated trajectory distribution $\bar{p}(\tau)=\mathcal{N}(\tau | \bar{\mu}_\tau, \bar{\Sigma}_\tau)$ is fitted using 

\begin{equation}
\label{eq:12}
\bar{\mu}_{\tau} = \frac{1}{N}\sum_{i=1}^N \bar{\tau}^i, \,\, \bar{\Sigma}_\tau = \frac{1}{N}\sum_{i=1}^N (\bar{\tau}^i - \bar{\mu}_\tau)(\bar{\tau}^i - \bar{\mu}_\tau)^T,
\end{equation}

\noindent The last step is to update the policy. We will only be updating the feed-forward terms $k_t$ and the covariance $C_t$. The feedback terms $K_t$ is kept constant (the policy parameters are $\theta_t = (k_t, C_t)$, $t \in [0,T)$). This significantly reduces the number of parameters to be updated and increases the learning speed. For each sample trajectory, we obtain its probability under $\bar{p}(\tau)$

\begin{equation}
p(\tau^i) = \mathcal{N}(\tau^i |\bar{\mu}_{\tau}, \bar{\Sigma}_{\tau})
\end{equation}

\noindent ( $p(\tau^i)$ is also written in short as $p^i$) where $i \in [1,N]$ is the sample index. Using these probabilities, a normalized weighting for each sample trajectory is calculated using the softmax function $w^i = e^{\alpha p^i}/\sum_{i=1}^N e^{\alpha p^i}$ ($\alpha>0$ is a parameter to be tuned). Finally, similar to~ \cite{chebotar2016path}, the policy is updated using weighted maximum likelihood by

\begin{equation}
\begin{split}
&k_t^\prime = \sum_{i=1}^N w^i k_t^i \\
&C_t^\prime = \sum_{i=1}^N w^i (k_t^i - k_t^\prime)(k_t^i - k_t^\prime)^T.
\end{split}
\end{equation}
\noindent  According to~\cite{stulp2012path}, such update strategy will result in convergence. The complete algorithm is described in Algorithm~\ref{alg:1}.

\begin{algorithm}
\caption{Temporal Logic Policy Search}
\label{alg:1}
\begin{algorithmic}[1]
\State \textbf{Inputs}: Episode horizon $T$, batch size $N$, KL constraint parameter $\epsilon$, smoothed robustness function $\hat{\rho}(s_{0:T},\phi)$, softmax parameter $\alpha>0$
\State Initialize policy $\pi \leftarrow (K_t, k_t, C_t)$
\State Initialize trajectory buffer $\mathcal{B} \leftarrow \emptyset$
\For{$m=1$ to \textit{number of training episodes}}
\State $\tau_m$ = SampleTrajectories($\pi,T$)
\State Store $\tau_m$ in $\mathcal{B}$
\If{Size($\mathcal{B}$) $\geq$ $N$}
\State $\bar{\tau}^i \leftarrow$ GetUpdatedTrajectories($\tau^i$) \textbf{for} $i=1$ to $N$ \textbf{end for} \Comment Using Equation~\eqref{tlps-eq-1}
\State $\bar{\mu}_\tau, \bar{\Sigma}_\tau\leftarrow$ FitTrajectoryDistribution($\{\tau_1, ..., \tau_N\}$) \Comment Using Equation~\eqref{eq:12}
\For{i=1 to N}
\State $p^i \leftarrow \mathcal{N}(\tau^i | \bar{\mu}_{\tau}, \bar{\Sigma}_{\tau})$ 
\State $w^i = \frac{e^{\alpha p^i}}{\sum_{i=1}^N e^{\alpha p^i}}$
\EndFor
\For{t = 0 to T-1}
\State $k_t^\prime \leftarrow \sum_i^N w^i k_t^i$
\State $C_t^\prime \leftarrow \sum_i^N w^i (k_t^i - k_t^\prime)(k_t^i - k_t^\prime)^T$
\EndFor
\State Clear buffer $\mathcal{B} \leftarrow \emptyset$
\EndIf
\EndFor
\end{algorithmic}
\end{algorithm}

\section{Robustness Smoothing}
\label{sec:smoothing}
In the TLPS algorithm introduced in the previous section, one of the steps requires solving a constrained optimization problem that maximizes the robustness (Equation~\eqref{tlps-eq-1}). The original robustness definition in Section~\ref{sec:tltl} is non-differentiable and thus rules out many efficient gradient-based methods. In this section we adopt a smooth approximation of the robustness function using \textit{log-sum-exp}. Specifically 

\begin{equation}
\label{eq:21}
\begin{split}
&\max(x_1, ..., x_n) \approx \frac{1}{\beta}\log \sum_i^n \exp(\beta x_i)\\
&\min(x_1, ..., x_n) \approx -\frac{1}{\beta}\log \sum_i^n \exp(-\beta x_i),
\end{split}
\end{equation}
where $\beta>0$ is a smoothness parameter. 
We denote an iterative max-min function as
\begin{equation*}
M(x)=\mami_{i}f_{i}(x),
\end{equation*}
where $f_{i}(x)=\mami_{j}f_{j}(x)$. $\mami$ denotes a function as $\mami\in\{\max,\min,\mathcal{I}\}$ where $\mathcal{I}$ is a operator such that $\mathcal{I}f_j(x)=f_{j}(x)$.  $i$ and $j$ are index of the functions in $\mami$ and can be any positive integer.
As we showed in Section~\ref{sec:tltl}, any robustness function could be expressed as an iterative max-min function.

Following the \textit{log-sum-exp} approximation, any iterative max-min function (i.e., the robustness of any TL formula) can be approximated as follows
\begin{equation*}
\hat{M}(x)=\frac{1}{\beta}\log{\left(\sum_{i}\exp{(\beta f_{i}(x))}\right)},
\end{equation*}
where $\beta_i>0$ if $\mami_{i}=\max_i$ and $\beta_i<0$ if $\mami_{i}=\min_{i}$.  In the reminder of this section, we provide three lemmas that show the following:
\begin{itemize}
\item the approximation error between $M(x)$ and $\hat{M}(x)$ approaches zero as $\beta_i \rightarrow \infty$. This error is always bounded by the $\log$ of the number of $f(x)$ which is determined by the number of predicates in the TL formulae and the horizon of the problem.  Tuning $\beta_i$ trades off between differentiability of the robustness function and approximation error. 
\item despite the error introduced by the approximation, the optimal points remain invariant (\emph{i.e.} $\argmax_{x}M(x)=\argmax_{x}\hat{M}(x)$). This result provides guarantee that the optimal policy is unchanged when using the approximated TL reward,
\item even though the \textit{log-sum-exp} approximation smooths the robustness function. Locally the ascending directions of $M(x)$ and $\hat{M}(x)$ can be tuned to coincide with small error and the deviation is controlled by the parameter $\beta$. As many policy search methods are local methods that improve the policy near samples, it is important to ensure that the ascending direction of the approximated TL reward does not oppose that of the real one.
\end{itemize}

\noindent Due to space constraints, we will only  provide sketches of the proofs for the lemmas.

\begin{lemma}
\label{lemma:1}
Let $N_{i}$ be the number of terms of $\mami_i$,  $M$ and $\hat{M}$ satisfy
\begin{equation*}
M-\sum_{i\in S_{min}}\frac{1}{|\beta_{i}|}\log{N_{i}}\leq \hat{M} \leq M + \sum_{i\in S_{max}}\frac{1}{\beta_{i}}\log{N_{i}}
\end{equation*}
where $S_{min}=\{i:\mami_{i}=\min_{i}\}$ and $S_{max}=\{i:\mami_{i}=\max_{i}\}$.
\end{lemma}

\begin{proof}
For simplicity and without loss of generality, we illustrate the proof of Lemma~\ref{lemma:1} by constructing an approximation for a finite \emph{max-min-max} problem
\begin{equation*}
\Phi(x)=\max_{i\in I}\min_{j\in J}\max_{k\in K}f_{i,j,k}(x).
\end{equation*}
Let $M_I=|I|$, $M_J=|J|$, $M_K=|K|$, and $\beta_I>0$, $\beta_J<0$, $\beta_K>0$. Firstly, we define $\Phi_{j}(x)=\max_{k\in K}f_{i,j,k}(x)$. Straightforward algebraic manipulation reveals that
\begin{align}
\label{maxEqu}
&\log\left(\sum_{j\in J}\exp(\beta_J\Phi_j)\right)+\frac{\beta_J}{\beta_K}\log(M_K)\\\nonumber
&\leq\log\left(\sum_{j\in J}\left[\sum_{k\in K}\exp(\beta_Kf_{i,j,k}(x))\right]^{\frac{\beta_J}{\beta_K}}\right)\\\nonumber
&\leq \log\left(\sum_{j\in J}\exp(\beta_J\Phi_j)\right).
\end{align}
Furthermore, let us define $\Phi_i=\min_{j\in J}\Phi_j$,  we have
\begin{align*}
&\beta_J\Phi_i\leq \log\left(\sum_{j\in J}\exp(\beta_J\Phi_j)\right)\leq\log(M_J)+\beta_J\Phi_i.
\end{align*}
By substituting into Equation~\eqref{maxEqu}, we obtain
\begin{align*}
\beta_J\Phi_i+\log(M_J)&\geq \log\left(\sum_{j\in J}\exp(\beta_J\Phi_j)\right)\\
&\geq \beta_J\Phi_i+\frac{\beta_J}{\beta_K}\log(M_K).
\end{align*}
Multiplying $\frac{1}{\beta_J}$ on both side, then
\begin{align*}
&\log(\sum_{i\in I}\exp(\beta_I\Phi_i))+\frac{\beta_I}{\beta_J}\log(M_J)\\
&\leq
\log\left(\sum_{i\in I}\left[\sum_{j\in J}\left(\sum_{k\in K}\exp(\beta_Kf_{i,j,k}(x))\right)^{\frac{\beta_J}{\beta_K}}\right]^{\frac{\beta_I}{\beta_J}}\right)\\
&~~~~~\leq\log(\sum_{i\in I}\exp(\beta_I\Phi_i))+\frac{\beta_I}{\beta_K}\log(M_K).
\end{align*}
Finally, let $\Phi=\max_{i\in I}\Phi_i$, then we have
\begin{align}
\nonumber
&\exp(\beta_I\Phi)\leq\sum_{i\in I}\exp(\beta_I\Phi_i)\leq M_I\exp(\beta_I\Phi)\\\label{minEqu}
&\beta_I\Phi\leq\log(\sum_{i\in I}\exp(\beta_I\Phi_i))\leq \log(M_I)+\beta_I\Phi
\end{align}
Substitute into Equation~\eqref{minEqu}
\begin{align*}
&\beta_I\Phi+\frac{\beta_I}{\beta_J}\log(M_J)\\
&\leq\log\left(\sum_{i\in I}\left[\sum_{j\in J}\left(\sum_{k\in K}\exp(\beta_Kf_{i,j,k}(x))\right)^{\frac{\beta_J}{\beta_K}}\right]^{\frac{\beta_I}{\beta_J}}\right)\\
&\leq\beta_I\Phi+\log(M_I)+\frac{\beta_I}{\beta_K}\log(M_K).
\end{align*}
Then we conclude the proof.
\end{proof}

\begin{lemma}
\label{lemma:2}
Suppose $X^*=\{x^*:x^*\in\argmax_{x}{M}(x)\}$, there exist a positive constant $B$ such that for all $|\beta|\geq B$ $x^*$ is also one of the maximum point of $\hat{M}(x)$ for any $x^*$, \emph{i.e.}
\begin{equation*}
x^*\in\argmax_{x}\hat{M}(x).
\end{equation*}

\end{lemma}
\begin{proof}
We start by considering $M$ as a maximum function, i.e. $M(x)=\max_{i}f_{i}(x)$.let us denote $I_{max}=\argmax_{i}f_{i}(x^*)$, then $x^*\in\argmax_{x}\hat{M}(x)$ when
\begin{align*}
&\sum_{i\neq I_{max}}\exp{(\beta f_{i}(x^*))}-\sum_{i\neq I_{max}}\exp{(\beta f_{i}(x))}\\
&\leq
\exp{(\beta f_{Imax}(x^*))}-\exp{(\beta f_{Imax}(x))}.
\end{align*}
There always exists a positive constant $B$, such that for all $\beta>B$ the above statement holds. Lemma~\ref{lemma:2} can be obtained by using the above proof for the $\mami$ function in general.
\end{proof}

\begin{lemma}
\label{lemma:3}
Let us denote the sub-gradient of $M$ as $\frac{\partial{M}}{\partial{x}}=\{\frac{\partial{M}}{\partial{x_1}},\ldots,\frac{\partial{M}}{\partial{x_N}}\}$ and the gradient of $\hat{M}$ as $\frac{\partial{\hat{M}}}{\partial{x}}=\{\frac{\partial{\hat{M}}}{\partial{x_1}},\ldots,\frac{\partial{\hat{M}}}{\partial{x_N}}\}$. There exists a positive constant $B$ such that for all $|\beta|\geq B$, $\frac{\partial{M}}{\partial{x}}$ and $\frac{\partial{\hat{M}}}{\partial{x}}$ satisfy
\begin{equation*}
\langle\frac{\partial{M}}{\partial{x}},\frac{\partial{\hat{M}}}{\partial{x}}\rangle\geq 0,
\end{equation*}
where $\langle\cdot,\cdot\rangle$ denotes the inner product.
\end{lemma}
\begin{proof}
Here we will only provide the proof when $M$ is a point-wise maximum of convex functions. One can generalize it to any iterative max-min function using the chain rule. Supposing $M(x)=\max_{i}f_{i}(x)$, the sub-gradient of $M(x)$ is 
\begin{equation*}
\frac{\partial M}{\partial x}=\partial f_{i}(x),i\in I(x),
\end{equation*}
where $I(x)=\{i|f_{i}(x)=f(x)\}$ is the set of ''active'' functions.
The corresponding $\hat{M}$ is defined as
\begin{equation*}
\hat{M}=\frac{1}{\beta}\log\left(\sum_{i}\exp{(\beta f_{i}(x))}\right),
\end{equation*}
where its first order derivative is
\begin{equation*}
\frac{\partial\hat{M}}{\partial x}=\sum_{i}\frac{\exp(\beta f_{i}(x))\partial f_{i}(x)}{\sum_{k}\exp{(\beta f_{k}(x))}}.
\end{equation*}
$\langle\frac{\partial M}{\partial x},\frac{\partial \hat{M}}{\partial x}\rangle>0$ if
\begin{align*}
&\frac{\exp{(\beta f_{i}(x))}}{\sum_{k}\exp{(\beta f_{k}(x))}}f_{i}(x)\\
&\geq\sum_{j\notin I(x)}\frac{\exp{(\beta f_{j}(x))}}{\sum_{k}\exp{(\beta f_{k}(x))}}f_{j}(x),\forall i\in I(x).
\end{align*}
Therefore, there always exists a positive constant $B$, such that $\langle\frac{\partial M}{\partial x},\frac{\partial \hat{M}}{\partial x}\rangle>0$ holds for all $\beta>B$.
\end{proof}


\section{Case Studies}
\label{sec:6}

In this section, we apply TLPS on a vehicle navigation example. As shown in Figure 1, the vehicle navigates in a 2D environment. It has a 6 dimensional continuous state feature space $s=[x,y,\theta, \dot{x}, \dot{y}, \dot{\theta}]$ where $(x,y)$ is the position of its center and $\theta$ is the angle its heading makes with the $x$-axis. Its 2 dimensional action space $a=[a_v, a_\Phi]$ consists of the forward driving speed and the steering angle of its front wheels. The car moves according to dynamics

\begin{equation}
\begin{split}
\dot{x} & = a_v \cos \theta \\
\dot{y} & = a_v \sin \theta \\
\dot{\theta} &= \frac{a_v}{L}\tan a_\Phi
\end{split}
\end{equation}
\noindent with added Gaussian noise ($L$ is the distance between the front and rear axles). However the learning agent is not provided with this model and needs to learn the desired control policy through trial-and-error.
\begin{figure}[tbh]
\label{fig:1}
\begin{center}
\includegraphics[width=1.\linewidth]{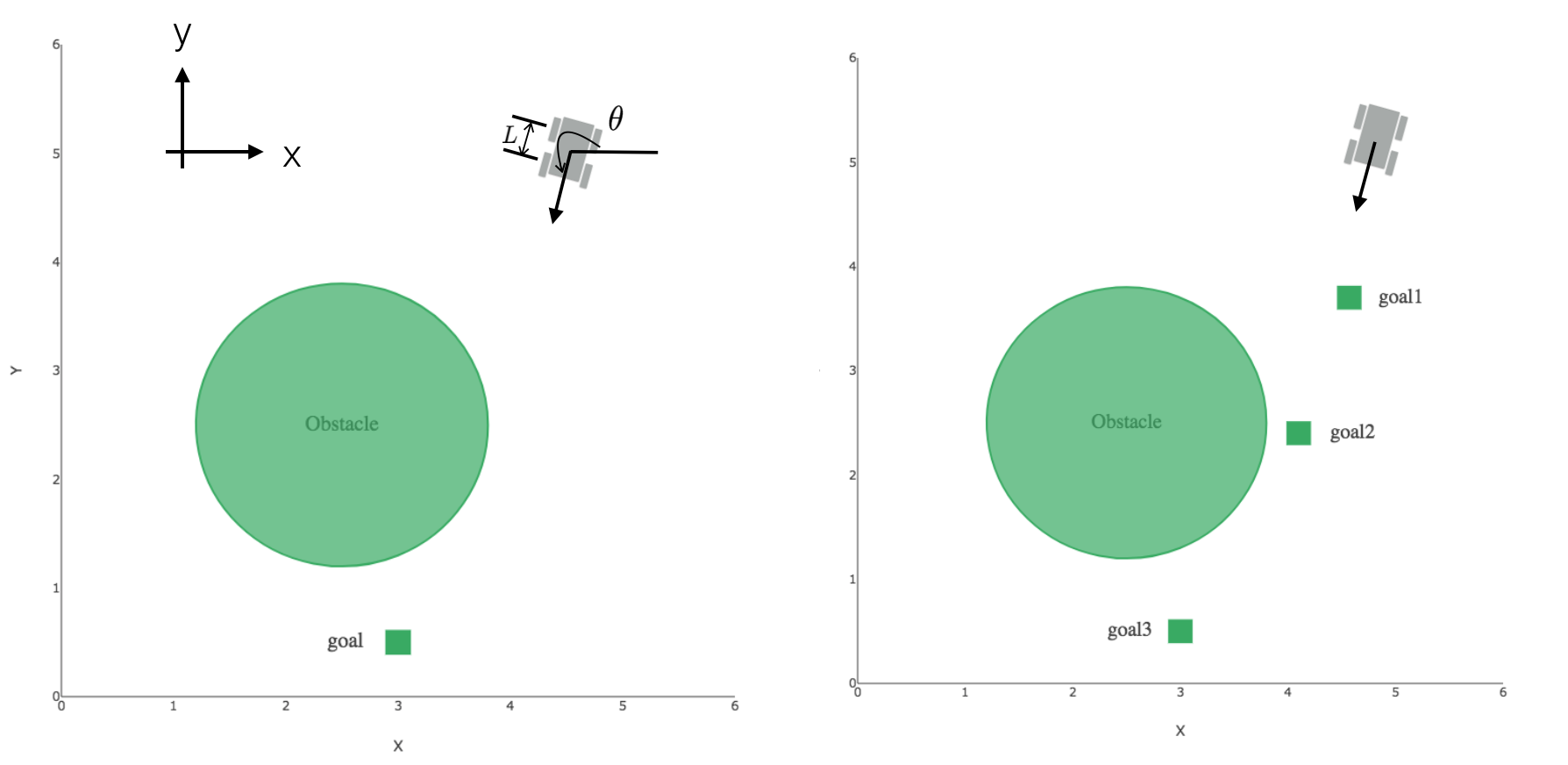}
\caption{Vehicle navigation task using TLTL specifications. The vehicle is shown in brown, the obstacle is shown as the green circle and the goals are shown as the green squares. \textit{left}: Task 1 is to reach the goal while avoiding the obstacle. \textit{right}: Task 2 is to visit goals 1,2,3 in this order while avoiding the obstacle}
\end{center}
\end{figure}

We test TLPS on two tasks with increasing difficulty. In the first task, the vehicle is required to reach the goal $g$ while avoiding the obstacle $o$. We express this task as a TLTL specification 

\begin{equation}
\label{eq:16}
\begin{split}
\phi_1 = &\Event(x > x_g^{l} \land x < x_g^{u} \land y > y_g^{l} \land y < y_g^{u}) \land \\ &\Always(d_{o} > r_{o} ).
\end{split}
\end{equation}
\noindent In Equation~\eqref{eq:16},  $(x_g^{l}, x_g^{u}, y_g^{l}, y_g^{u})$ defines the square shaped goal region, $d_o$ is the Euclidean distance between the vehicle's center and the center of the obstacle, $r_o$ is the radius of the obstacle. In English, $\phi_1$ describes the task of "\textit{eventually} reach goal $\boldsymbol g$ \textit{and always} stay away from the obstacle". Using the quantitative semantics described in Section~\ref{sec:tltl} , the robustness of $\phi_i$ is 

\begin{equation}
\label{eq:17}
\begin{split}
&\rho_1(\phi_1, (x, y)_{0:T}) =\\
&\min\Bigg(\underset{t \in [0,T)}{\max}\bigg( \min\Big( x_t - x_g^{l}, x_g^{u} - x_t, y_t - y_g^{l}, y_g^{u} - y_t\Big)\bigg), \\
&\underset{t\in [0,T)}{\min}\Big(d^t_{o}-r_{o}\Big)\Bigg),
\end{split}
\end{equation}

\noindent where $(x_t, y_t)$ and $d_o^t$ are the vehicle position and distance to obstacle center at time $t$. Using the \textit{log-sum-exp}, approximation for $\rho_1(\phi_1, (x, y)_{0:T}) $ can be obtained as 

\begin{equation}
\label{eq:18}
\begin{split}
&\hat{\rho}_1(\phi_1, (x, y)_{0:T}) = \\
& -\frac{1}{\beta}\log \sum_{t=0}^T\Big ( \exp[-\beta( x_t - x_g^{l})] +  \exp[-\beta( x_g^{u} - x_t)] + \\
& \exp[-\beta(y_t - y_g^{l})] + \exp[-\beta( y_g^{u} - y_t)] + \exp[-\beta(d_o^t - r_o)]\Big).
\end{split}
\end{equation}
\noindent Because we used the same $\beta$ throughout the approximation, intermediate $\log$ and $\exp$ cancel and we end up with Equation~\eqref{eq:18}. $\hat{\rho}_1(\phi_1, (x, y)_{0:T})$ is used in the optimization problem defined in Equation~\eqref{tlps-eq-1}.

In task 2, the vehicle is required to visit goals 1, 2, 3 in this specific order while avoiding the obstacle. Expressed in TLTL results in the specification

\begin{equation}
\label{eq:19}
\begin{split}
\phi_2=&(\psi_{g_1} \,\, \mathcal{T} \,\, \psi_{g_2} \,\, \mathcal{T} \,\, \psi_{g_3}) \land  (\neg(\psi_{g_2} \lor \psi_{g_3}) \,\, \mathcal{U} \,\, \psi_{g_1}) \land \\
&(\neg(\psi_{g_3}) \,\, \mathcal{U} \,\, \psi_{g_2}) \land (\underset{i=1,2,3}{\bigwedge} \Box(\psi_{g_i} \Rightarrow \bigcirc \Box \neg \psi_{g_i})) \land \\ 
&\Box (d_o > r_o),
\end{split}
\end{equation}

\noindent where $\bigwedge$  is a shorthand for a sequence of conjunction, $\psi_{g_i}:x > x_{g_i}^{l} \land x < x_{g_i}^{u} \land y > y_{g_i}^{l} \land y < y_{g_i}^{u}$ are the predicates for goal $g_i$. In English, $\phi_2$ states "visit $g_1$ \textit{then} $g_2$ \textit{then} $g_3$, \textit{and don't} visit $g_2$ \textit{or} $g_3$ \textit{until} visiting $g_1$, \textit{and don't} visit $g_3$  \textit{until} visiting $g_2$, \textit{and always} if visited $g_i$ \textit{implies next always don't} visit $g_i$ (don't revisit goals), \textit{and always} avoid the obstacle" . Due to space constraints the robustness of $\phi_2$ and its approximation will not be explicitly presented, but it will take a similar form of nested $\min()/\max()$ functions that can be generated from the quantitative semantics of TLTL.

During training time, the obstacle is considered "penetrable" in that the car can surpass its boundary with a negative reward granted according to the penetrated depth. In practice we find that this facilitates learning compared to a single negative reward given at contact with the obstacle and restarting the episode.  

Each episode has a horizon $T=200$ time-steps. 40 episodes of sample trajectories are collected and used for each update iteration. The policy parameters are initialized randomly in a given region (the policy covariances should be initialized to relatively high values to encourage exploration). Each task is trained for 50 iterations and the results are presented in Figures 2 and 3. Figure 2 shows sample trajectory distributions for selected iterations. Trajectory distributions are illustrated as shaded regions with width equal to 2 standard deviations. Lighter shade indicates earlier time in the training process. We used $\beta=9$ for this set of results. We can see from Figure 2 that the trajectory distribution is able to converge and satisfy the specification. Satisfaction occurs much sooner for task 1 (around 30 iterations) compared with task 2 (around 50 iterations).  

\begin{figure*}
 \centering
 \begin{multicols}{2}
  \includegraphics[width=2.0\columnwidth]{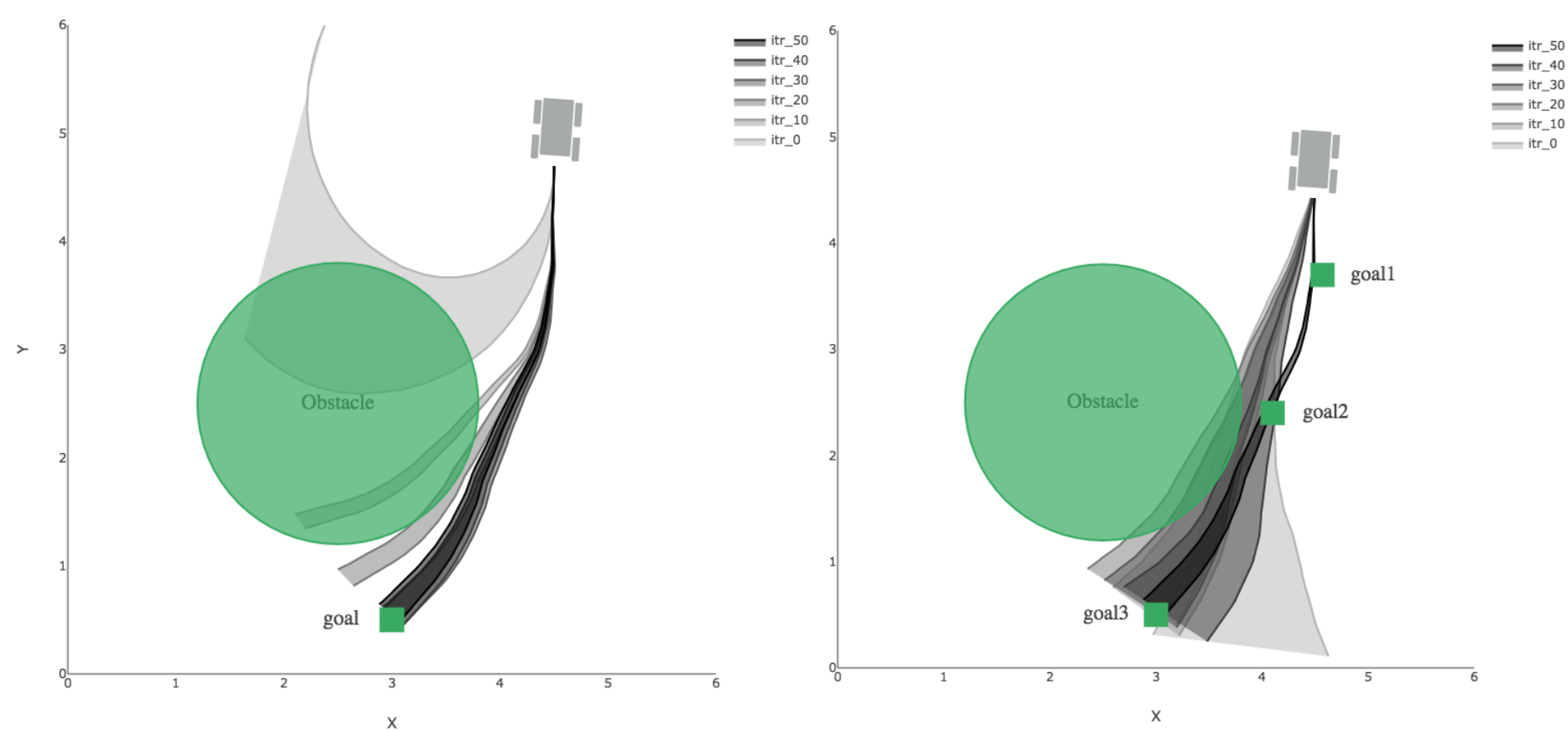}\label{fig:2}
  \end{multicols}
  \caption{Sample trajectory distributions for selected iterations for \textit{left}: task 1, \textit{right}: task 2. Each iteration consists of 40 sample trajectories each having a horizon of 200 time-steps. The width of each distribution is 2 standard deviations and color represent recency in the training process (lighter color indicates earlier time in training).}
\end{figure*}

\begin{figure*}
 \centering
 \begin{multicols}{2}
  \includegraphics[width=2.0\columnwidth]{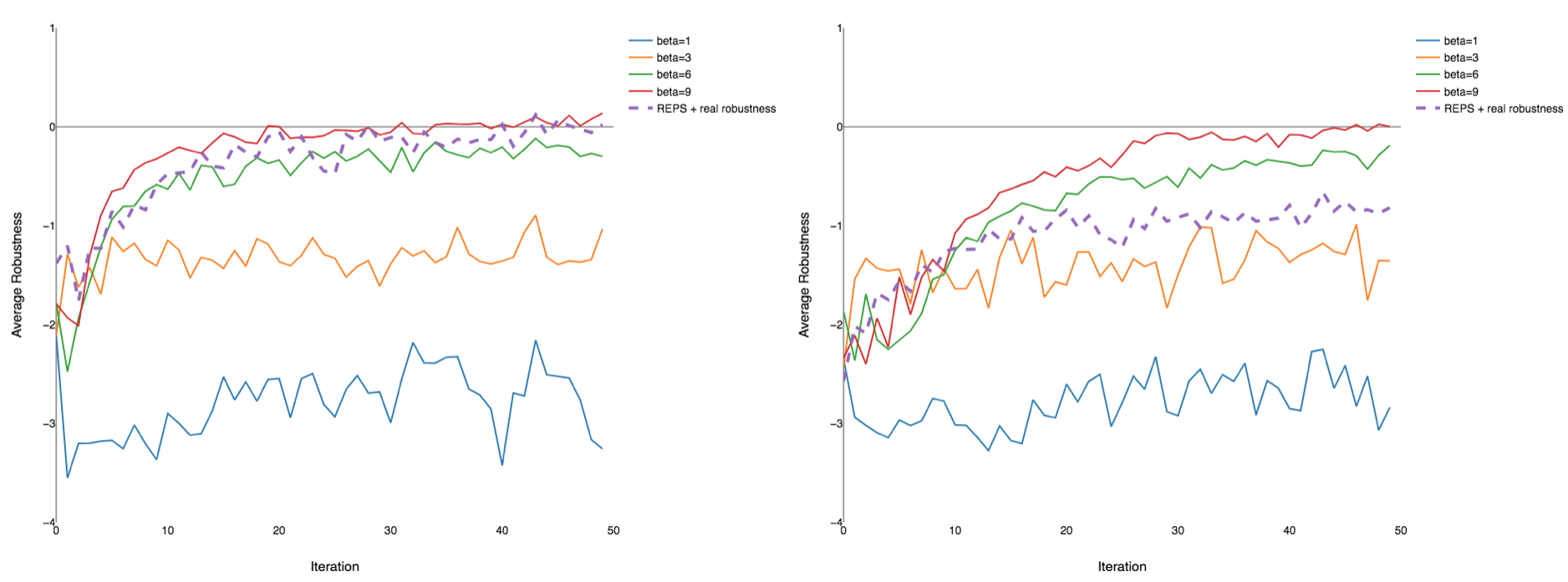}\label{fig:3}
  \end{multicols}
  \caption{Average return vs training iteration for \textit{left}: task 1, \textit{right}: task2. The average return is represented as the original robustness value calculated from sample trajectories. TLPS is compared with varying $\beta$. REPS with the original robustness as terminal reward is used as a baseline. }
\end{figure*}

Figure 3 compares the average robustness (of 40 sample trajectories) per iteration for TLPS with different values of the approximation parameters $\beta$  in~\eqref{eq:21}. As a baseline, we also compare TLPS with episode-based relative entropy policy search (REPS)~\cite{Deisenroth2011}. The original robustness function is used as the terminal reward for REPS and our previous work~\cite{li2016reinforcement} has shown that this combination outperforms heuristic reward designed for the same robotic control task.  The magnitude of robustness value changes with varying $\beta$. Therefore, in order for the comparison to be meaningful (putting average returns on the same scale), sample trajectories collected for each comparison case are used to calculate their original robustness values against the TLTL formula and plotted in Figure 3 (a similar approach taken in~\cite{li2016reinforcement}). The original robustness is chosen as the comparison measure for its semantic integrity (value greater than zero indicates satisfaction).

Results in Figure 3 shows that larger $\beta$ results in faster convergence and higher average return. This is consistent with the results of Section~\ref{sec:smoothing} since larger $\beta$ indicates lower approximation error. However, this advantage diminishes as $\beta$ increases due to the approximated robustness function losing differentiability. For the relatively easy task 1, TLPS performed comparatively with REPS. However, for the harder task 2, TLPS exhibits a clear advantage both in terms of rate of convergence and quality of the learned policy.


TLPS is a local policy search method that offers gradual policy improvement, controllable policy space exploration and smooth trajectories. These characteristics are desirable for learning control policies for systems that involve physical interactions with the environment. S (likewise for other local RL methods). Results in Figure 3 show a rapid exploration decay in the first 10 iterations and little improvement is seen after the $40^{th}$ iteration. During experiments, the authors find that adding a policy covariance damping schedule can help with initial exploration and final convergence. A principled exploration strategy is possible future work. 

Similar to many policy search methods, TLPS is a local method. Therefore, policy initialization is a critical aspect of the algorithm (compared with value-based methods such as Q-learning). In addition, because the trajectory update step in Equation~\eqref{tlps-eq-1} does not consider the system dynamics and relies on being close to sample trajectories, divergence can occur with a small $\beta$ or a large learning rate. Making the algorithm more robust to hyperparameter changes is also an important future direction.  

\section{Conclusion} 
\label{sec:conclusion}
As reinforcement learning research advance and more general RL agents are developed, it becomes increasingly important that we are able to correctly communicate our intentions to the learning agent. A well designed RL agent will be proficient at finding a policy that maximizes its returns, which means it will exploit any flaws in the reward function that can help it achieve this goal. Human intervention can sometimes help alleviate this problem by providing additional feedback. However, as discussed in~\cite{dewey2014reinforcement}, if the communication link between human and the agent is unstable (space exploration missions) or the agent operates on a timescale difficult for human to respond to (financial trading agent), it is critical that we are confident about what the agent will learn.

In this paper, we applied temporal logic as the task specification language for reinforcement learning. The quantitative semantics of TL is adopted for accurate expression of logical relationships in an RL task. We explored robustness smoothing as a means to transform the TL robustness to a differentiable function and provided theoretical results on the properties of the smoothed robustness. We proposed temporal logic policy search (TLPS), a model-free method that utilizes the smoothed robustness and operates in continuous state and action spaces. Simulation experiments are conducted to show that TLPS is able to effectively find control policies that satisfy given TL specifications. 


\bibliographystyle{IEEEtran}
\bibliography{references}

\end{document}